\documentclass[11pt,a4paper]{article}

\usepackage{hyperref}
\usepackage[margin=1in]{geometry}
\usepackage{array}
\usepackage[
    type={CC},
    modifier={by-nc-nd},
    version={4.0},
]{doclicense}

\usepackage[utf8]{inputenc}
\usepackage[T1]{fontenc}
\usepackage{url}
\usepackage{booktabs}
\usepackage{amsmath,amsthm,amssymb,amsfonts}
\usepackage{nicefrac}
\usepackage{microtype}
\usepackage{xcolor}
\usepackage{enumitem}

\newtheorem{theorem}{Theorem}
\newtheorem{lemma}[theorem]{Lemma}
\newtheorem{proposition}[theorem]{Proposition}
\newtheorem{corollary}[theorem]{Corollary}
\theoremstyle{definition}
\newtheorem{definition}[theorem]{Definition}
\newtheorem{axiom}{Property}
\theoremstyle{remark}
\newtheorem{remark}[theorem]{Remark}

\newcommand{\Env}{\mathbf{Env}}
\newcommand{\IC}{\mathrm{IC}}

\newcommand{\UTM}{\mathcal{U}}

\title{Intervention Complexity as a Canonical Reward and a Measure of Intelligence}

\author{%
  Brendan McCane\texorpdfstring{\\}{ }
  School of Computing\texorpdfstring{\\}{ }
  University of Otago\texorpdfstring{\\}{ }
  Dunedin, New Zealand
}

\begin{document}

\maketitle

\begin{abstract}
The Legg--Hutter universal intelligence measure assesses general
intelligence as expected reward across all computable environments,
weighted by simplicity. The measure presupposes an externally specified
reward function, raising the question of whether a canonical choice
exists. We propose intervention complexity (IC) as such a choice. Given
a resource function $\rho$ encoding an inductive bias such as program
length or action count, $\rho$-IC is the minimum $\rho$-cost of a
program achieving a given state transition. We identify five natural
properties of IC that distinguish it from existing reward measures. We
propose a two-dimensional characterisation of intelligence: agent
competence and learning efficiency. A separation theorem establishes
that the choice of resource bias determines computability:
action-count IC is polynomial-time computable, while program-length IC
without oracle access is uncomputable. The gap between bare and oracle
IC quantifies the information content of learning the environment.
\end{abstract}

\section{Introduction}\label{sec:intro}

The Legg--Hutter universal intelligence measure~\cite{legg2007} defines
the intelligence of agent~$\pi$ as
\begin{equation}\label{eq:LH}
  \Upsilon(\pi) \;=\; \sum_{\mu} 2^{-K(\mu)}\, V_{\mu}^{\pi},
\end{equation}
the sum over all computable environments of expected reward, weighted by
environment simplicity (with $K(\mu)$ the Kolmogorov complexity of $\mu$).
The measure is a formalisation of intelligence as general goal
achievement, derived from the AIXI agent
model~\cite{hutter2000,hutter2005}.

The measure is precise but presupposes a reward function $V_\mu^\pi$
which is left as a free parameter. The original framework constrains
this only by boundedness and computability~\cite{legg2007}. Two identical
environments can therefore be assigned different rewards depending on
the designer's choice, and intelligence in this definition is defined
relative to a reward signal that may itself be pathological. This
indeterminacy has been recognised in the
literature~\cite{hernandez2017,chollet2019} but a canonical choice has
not been given.

We propose intervention complexity (IC) as a canonical family of rewards that is
derived from environment structure alone albeit parameterised by a preferred resource bias. The intuition is that
intelligence manifests as efficient causal intervention: given a state
$s$ and a target $s'$, an agent's task is to find a short program
producing an action sequence that effects the transition. Different
notions of ``short'' yield different measures, captured through a
resource function $\rho$ that encodes the inductive bias.

\textbf{Contributions.}
\begin{itemize}[leftmargin=*]
  \item We define intervention complexity, a canonical reward
    parameterised by resource bias, derivable from environment
    structure without external normative input
    (Section~\ref{sec:framework}).
  \item We identify five natural properties of IC that distinguish it
    from existing reward measures, with each property non-trivially
    selecting against existing measures from the literature
    (Section~\ref{sec:properties}).
  \item We propose agent competence as a Dimension~1 measure of
    intelligence, which combines task difficulty with execution
    efficiency in a way that avoids savant pathologies
    (Section~\ref{sec:dimensions}).
  \item We prove a separation theorem on the computability of IC under
    different resource biases (Section~\ref{sec:separation}). The
    bare-versus-oracle IC gap quantifies the information content of
    learning, and the action-count case shows that IC is tractable for
    finite environments.
  \item We discuss implications for superintelligence
    (Section~\ref{sec:discussion}). The framework predicts a logistic
    rather than exponential trajectory of intelligence improvement, with
    the oracle competence as a fixed ceiling. Under Solomonoff
    weighting, simple environments dominate the scalar measure, where
    human civilisation has had longest to develop accurate models.
\end{itemize}

\textbf{Related work.}
Our work extends the Legg--Hutter universal intelligence
measure~\cite{legg2007}, itself derived from Hutter's AIXI
framework~\cite{hutter2005,hutter2024}. The core question---whether the
reward primitive in Legg--Hutter can be grounded rather than
stipulated---has not, to our knowledge, been addressed before.

The closest contemporary work is Chollet's~\cite{chollet2019} definition of
intelligence as {skill-acquisition efficiency} over a scope of tasks,
accounting for priors, experience, and generalisation difficulty. Chollet's
definition is motivated by psychometrics and culminates in an empirical
benchmark (the Abstraction and Reasoning Corpus). Our Dimension~2 (learning
efficiency) formalises a similar intuition---intelligence as the rate of
learning rather than the level of task performance---but within a
framework that yields upper bounds on performance. The key differences are
that our framework 
(i)~separates the target of intelligence (Dimension~1) from
the process of acquiring it (Dimension~2), and (ii)~proves that the
choice of resource bias determines the computability of the resulting
measure.

Hern\'andez-Orallo and Dowe~\cite{hernandez2010} develop practical
universal intelligence tests grounded in algorithmic information theory and
the Solomonoff prior~\cite{solomonoff1964}. Hern\'andez-Orallo's
comprehensive treatment~\cite{hernandez2017} elaborates task difficulty as
a function of descriptional complexity---an idea closely related to our
intervention complexity. Our contribution differs in 
in that it does not propose a
specific test or difficulty measure, and does not require the use of arbitrary rewards.

Several lines of research propose intrinsic reward functions grounded in
environmental structure. Schmidhuber's formal theory of
creativity~\cite{schmidhuber2010} defines intrinsic reward as compression
progress. Friston's free energy principle~\cite{friston2010} identifies
intelligence with prediction-error minimisation. Klyubin, Polani, and
Nehaniv's empowerment~\cite{klyubin2005} defines a task-independent
intrinsic motivation as the channel capacity between an agent's actions
and its future sensor states---an information-theoretic measure of
causal influence that, like our IC, is derived entirely from the
environment's structure without external reward. The key difference is
that empowerment measures the {breadth} of an agent's causal
influence (how many distinguishable outcomes it can produce), while IC
measures the {efficiency} of directed intervention (how cheaply it
can reach a specific target). In our framework, empowerment
violates minimality (Property~\ref{ax:min}), as it depends on the full
action-to-observation channel rather than on the optimal intervention
for a specific transition. Our framework places all three
as Dimension~2 strategies (efficient heuristics for learning the world
model) rather than Dimension~1 definitions (the target of intelligence),
and the independence results of Section~\ref{sec:axioms}
make this distinction precise.

The connection between intervention efficiency and intelligence has
antecedents in Pearl's causal inference framework~\cite{pearl2009} and
in the study of planning complexity~\cite{bylander1994}. The separation
between program-length and execution-time biases connects to Bennett's
notion of logical depth~\cite{bennett1988}, which measures the
computational time required by the shortest program---an alternative
resource bias in our framework. Our formalisation
draws on algorithmic information theory~\cite{lv2008}, particularly the
invariance theorem for Kolmogorov complexity.

The link between intelligence and computational complexity has important
precursors. Simon's bounded rationality~\cite{simon1955} informally
identified computational constraints as the source of suboptimal
decision-making. Russell and Subramanian~\cite{russell1995} formalised
this as {bounded optimality}---an agent is bounded-optimal if its
program is the best solution to the constrained optimisation problem
presented by its architecture and task environment---and developed an
{asymptotic bounded optimality} (ABO) notion explicitly analogous
to big-$O$ notation in complexity theory.
Russell~\cite{russell1997} argued that bounded optimality is the right
formal definition of intelligence, connecting it directly to
computational complexity classes, however, it too relies on an external definition of a utility function (analogous to a reward function). Papadimitriou and
Yannakakis~\cite{papadimitriou1994} formalised complexity itself as a
model of bounded rationality. Our framework differs from this 
as we propose a canonical reward (IC) 
that is defined in all environments rather than taking reward as given, and we show that it has useful properties that distinguish it from previous reward functions.
The unified action-space observation of
Section~\ref{subsec:competence} on agent competence, can be seen as analogous to Russell's bounded optimality programme.

\section{Framework}\label{sec:framework}

For this work, we limit ourselves to discrete and deterministic environments, and leave continuous and stochastic environments for future work. As, such, first, fix a universal Turing machine~$\UTM$.

\begin{definition}[Computable environment]
A \emph{computable environment} is a tuple $\mu = (S_\mu, A_\mu, T_\mu)$ where
$S_\mu$ is a countable state space, $A_\mu$ a countable action set, and
$T_\mu : S_\mu \times A_\mu \to S_\mu$ a computable transition function.
We write $\Env$ for the class of computable environments. 
\end{definition}

\begin{definition}[Intervention]
An \emph{intervention} from $s$ to $s'$ in $\mu$ is a program $p$ that,
when executed on $\UTM$ outputs an action sequence $a_1, \ldots, a_n$ such that the
states reached by applying $T_\mu$ from $s$ end at $s'$. The
\emph{intervention set} $I_\mu(s, s')$ is the set of all such programs.
\end{definition}

\begin{definition}[Resource function]
A \emph{resource function} $\rho$ assigns a non-negative real cost to
each program. Examples include program length $\rho_\ell(p) = |p|$
(the inductive bias of Kolmogorov complexity); action count
$\rho_a(p) = $ length of $p$'s output (a count of physical actions);
execution time $\rho_t(p) = $ number of steps until $p$ halts (a
proxy for thinking time); and energy $\rho_e(p)$, summing
per-action energy costs. We require $\rho$ to be computable,
well-ordered, and non-trivial (it takes at least two distinct values
across $\Env$).
\end{definition}

\begin{remark}[Combined biases and Levin complexity]
The biases above can be combined. The combined bias
$\rho_{\mathrm{comb}}(p) = \alpha|p| + \beta|\mathrm{output}(p)|$ for
$\alpha, \beta > 0$ penalises both long programs (inelegant strategies)
and long action sequences (inefficient solutions). It is the
intervention analogue of Levin's $Kt$ complexity~\cite{levin1973}, which
combines program length with computation time. The combined bias is
needed to address the exhaustive-search problem: under program-length
bias alone, the shortest program may be a brute-force search that
outputs an arbitrarily long action sequence. The combined bias rules
this out by penalising the long output.
\end{remark}

\begin{definition}[Intervention complexity]\label{def:ic}
The \emph{$\rho$-intervention complexity} of the transition $s \to s'$
in environment $\mu$ is
\[
  \IC_\mu^\rho(s, s') = \min_{p \in I_\mu(s, s')} \rho(p, \mu, s),
\]
with $\IC = \infty$ if $I_\mu(s, s')$ is empty. A canonical reward is
then $R = g \circ \IC_\mu^\rho$ for some strictly increasing $g$, with
higher IC denoting greater achievement.
\end{definition}

\section{Properties of IC}\label{sec:properties}

We identify five properties of IC, each violated by other candidate
reward measures from the literature.

\textbf{P1 (Environment-derived):} $\IC_\mu^\rho(s, s')$ is computable
from $\mu, s, s'$ alone. Standard reinforcement learning rewards and
RLHF-learned preference models~\cite{christiano2017,rafailov2023} violate
this: rewards depend on external designer choice or human preference data
rather than on environment structure. Friston's free-energy
principle~\cite{friston2010} violates this in a different way, since
surprise-based reward depends on the agent's epistemic state rather than
the environment alone.

\textbf{P2 (Universal):} $\IC_\mu^\rho$ is defined for all computable
environments. Potential-based reward shaping~\cite{ng1999} violates
this: the choice of potential function $\Phi$ requires domain-specific
knowledge about which states are desirable, and there is no general
procedure assigning a meaningful $\Phi$ to every computable environment.

\textbf{P3 (Minimality):} $\IC_\mu^\rho$ depends on $I_\mu(s,s')$ only
through the minimum, discarding all suboptimal interventions.
Schmidhuber's compression progress~\cite{schmidhuber2010} violates this:
reward depends on the agent's entire learning history and full model of
the intervention set, not on the optimal intervention alone. Empowerment
likewise depends on the channel capacity between actions and states,
not on the cheapest path.

\textbf{P4 (Sensitivity):} Distinct minimum costs yield distinct IC
values. Sparse binary rewards (the standard in many RL benchmarks) and
threshold-based rewards violate this: they collapse all transitions
above some cost threshold to a single value. Constant rewards (such as
``survival bonus'' rewards) are the extreme case.

\textbf{P5 (Achievement preference):} Higher minimum cost yields higher
IC. The reward $-\IC$ violates this: it inverts the direction by
rewarding ease over achievement, a ``lazy reward'' that maximises by
doing nothing.
Achievement preference
encodes the substantive judgement that intelligence is demonstrated by
solving {harder} problems, not easier ones.

\begin{remark}
$\IC_\mu^\rho$ specifies the best any agent can do.
How close its actual cost is to the
IC minimum is captured by Dimension~2 (Section~\ref{sec:dimensions}), not
by the reward functional.
\end{remark}

There are two consequences of the above properties. The first is a direct consequence of the
non-triviality of $\rho$, there exist $\mu, s, s_1', s_2'$ with
$\IC_\mu^\rho(s, s_1') \neq \IC_\mu^\rho(s, s_2')$ and hence:
\begin{lemma}[Non-degeneracy]\label{lem:nondeg}
$\IC_\mu^\rho$ is non-degenerate:
there exist $\mu \in \Env$ and $s, s_1', s_2' \in S_\mu$
such that $\IC_\mu^\rho(s, s_1') \neq \IC_\mu^\rho(s, s_2')$.
\end{lemma}

\begin{lemma}[Ordinal invariance]\label{lem:inv}
$\IC_\mu^\rho$
is ordinally invariant under change of UTM. Specifically, 
the ranking induced by $\IC$ is robust to the choice of UTM, with the
strength of the guarantee depending on the resource bias:
\begin{enumerate}[label=(\roman*)]
  \item For intrinsic biases whose cost depends only on the output
    (action count, energy): $\IC$ does not depend on the UTM at all, and
    the ranking is exactly preserved.
  \item For program-length bias: the invariance theorem for Kolmogorov
    complexity~\cite{lv2008} gives an additive bound
    $|\IC_{\mu,\UTM_1}^{\rho_\ell} - \IC_{\mu,\UTM_2}^{\rho_\ell}| \leq c$
    for a constant $c$ depending only on the UTM pair. The ranking is
    preserved for all pairs of transitions whose IC values differ by more
    than $2c$.
  \item For execution-time bias: the simulation theorem gives a polynomial
    bound $\IC_{\mu,\UTM_2}^{\rho_t} \leq \mathrm{poly}(\IC_{\mu,\UTM_1}^{\rho_t})$.
    The ranking is preserved for all pairs of transitions whose IC values
    differ by more than a polynomial factor.
\end{enumerate}
\end{lemma}

\noindent Table~\ref{tab:independence} summarises the independence results.

\begin{table}[ht]
\centering
\small
\begin{tabular}{@{}lll@{}}
\toprule
\textbf{Property} & \textbf{Violated by} & \textbf{What is ruled out} \\
\midrule
1 (Environment-derived) & RL rewards, RLHF, surprise-based reward & External/agent-dependent input \\
2 (Universal) & Potential-based shaping & Domain dependence \\
3 (Minimality) & Intervention-set entropy, compression progress & Non-optimal features \\
4 (Sensitivity) & Sparse/binary rewards, constant reward & Collapsing distinct difficulties \\
5 (Achievement preference) & Ease-seeking reward ($R = -\IC$) & Rewarding ease over achievement \\
\bottomrule
\end{tabular}
\caption{Independence of the properties. Each row exhibits a reward
functional satisfying all other properties but violating the one listed.}
\label{tab:independence}
\end{table}

\section{Metric Structure of IC}\label{sec:noncomp}

We examine how IC behaves under composition of transitions, showing
that it is sub-additive but not additive, and that it forms a directed
quasimetric on the state space.

\begin{proposition}[Sub-additivity]\label{prop:subadd}
For all $\mu, s, s', s''$ with $s'$ reachable from $s$ and $s''$
reachable from $s'$,
\[
  \IC_\mu^\rho(s, s'') \;\leq\; \IC_\mu^\rho(s, s') + \IC_\mu^\rho(s', s'') + O(1),
\]
where the $O(1)$ constant accounts for the overhead of composing two
programs (and is zero for action-count bias).
\end{proposition}

\begin{proposition}[Non-additivity]\label{prop:nonadd}
There is no composition operation $\oplus$ such that for all $\mu, s, s', s''$,
\[
  \IC_\mu^\rho(s, s'') = \IC_\mu^\rho(s, s') \oplus \IC_\mu^\rho(s', s'').
\]
\end{proposition}

\begin{remark}[Directed quasimetric]
IC satisfies three of the four metric axioms: non-negativity
($\IC \geq 0$), identity ($\IC(s,s) = 0$), and the triangle inequality
(Proposition~\ref{prop:subadd}). It violates symmetry:
$\IC(s,s') \neq \IC(s',s)$ in general, as the non-additivity proof
illustrates (the return trip may be much harder than the outward
journey). IC is therefore a {directed quasimetric} on the state
space---a natural structure for environments with irreversible
transitions.

The combination of sub-additivity and non-additivity 
means the direct route from $s$ to $s''$ can be strictly cheaper than
any decomposition through intermediate states (non-additivity), but it
can never be more expensive (sub-additivity). An agent that discovers
direct shortcuts is more intelligent than one that routes through
intermediate states, and the $\IC$ correctly credits this.
\end{remark}

\section{Two Dimensions of Intelligence}\label{sec:dimensions}

The canonical reward characterises Dimension~1 (agent competence): how
well a fixed agent performs relative to the oracle benchmark. A real
agent must learn $T_\mu$ through experience, so its competence will
improve over time. The rate of improvement is Dimension~2 (learning
efficiency).

\subsection{Agent competence}\label{subsec:competence}

For an agent $\pi$ in environment $\mu$, let $C_\pi(\mu, s, s')$ be the
$\rho$-cost of the intervention $\pi$ produces, with $C_\pi = \infty$
if $\pi$ fails. A naive scalar measure suffers from a savant pathology:
an agent that solves one extremely hard task but fails all easy ones
would score highly under any difficulty-weighted scalar despite being
almost useless in practice. We therefore define competence as a function of
task difficulty.

\begin{definition}[Competence curve]
For environment $\mu$ and difficulty level $k \geq 0$, let
$R_\mu^k = \{(s,s') \in R_\mu : \IC_\mu^\rho(s,s') \leq k\}$ be the
reachable pairs with oracle competence at most $k$. The
\emph{competence curve} of agent $\pi$ at level $k$ in $\mu$ is
\[
  \Gamma_\rho(\pi, k, \mu)
  = \begin{cases}
    \displaystyle\frac{1}{|R_\mu^k|}
    \sum_{(s,s') \in R_\mu^k}
    \frac{\IC_\mu^\rho(s,s')^2}{C_\pi(\mu, s, s')}
    & \text{if } R_\mu^k \neq \emptyset, \\
    0 & \text{otherwise,}
  \end{cases}
\]
where $\IC^2/C = 0$ when $C = \infty$. Aggregated over environments:
$\Gamma_\rho(\pi, k) = \sum_\mu 2^{-K(\mu)} \Gamma_\rho(\pi, k, \mu)$.
\end{definition}

The per-task contribution $\IC^2/C = \IC \cdot (\IC/C)$ combines task
difficulty $\IC$ with efficiency $\IC/C$. An optimal agent ($C = \IC$)
contributes $\IC$, a wasteful agent contributes proportionally less,
and a failed task contributes zero. 

\begin{remark}[Connection to bounded optimality]
The competence curve $\Gamma_\rho(\pi, k)$ provides a concrete
instantiation of the bounded-optimality ideas of Simon~\cite{simon1955}
and Russell and Subramanian~\cite{russell1995,russell1997}. An agent's
curve traces out its effective capability as a function of task
difficulty where a sharp drop at some $k^*$ reveals the difficulty level
beyond which the agent's resources are insufficient, while a curve
that remains high to large $k$ indicates an agent with greater
effective resources. Russell's asymptotic bounded optimality
classifies agents by how performance scales with computational
budget; the competence curve is the analogous classification by how
competence scales with task difficulty. Two agents with identical
architectures but different competence curves have effectively
different bounded-optimal frontiers, since the curve reflects what
the agent achieves rather than what its architecture
nominally permits.
\end{remark}

\begin{definition}[Scalar agent competence]\label{def:scalar-comp}
When a single number is needed, define the
{$\rho$-agent competence} of $\pi$ as the Solomonoff-weighted
integral of the competence curve:
\begin{equation}\label{eq:agent-comp}
  \Gamma_\rho(\pi) \;=\;
  \int_{0}^{\infty} 2^{-k}\; \Gamma_\rho(\pi, k)\; dk.
\end{equation}
\end{definition}

The $2^{-k}$ weighting over difficulty levels plays the same role as the
$2^{-K(\mu)}$ weighting over environments where simpler tasks count more.
An agent that is competent on easy-to-moderate tasks scores well; an agent
that handles only exotic hard tasks scores poorly unless it also handles
the easy ones. This avoids the ``savant pathology'' in which solving one
extremely hard task outweighs failing everything else.

\begin{remark}[Connection to Legg--Hutter]
Scalar agent competence~\eqref{eq:agent-comp} can be understood as a
principled completion of the Legg--Hutter universal intelligence
measure~\eqref{eq:LH}. It replaces the externally specified reward
$V_\mu^\pi$ with the
difficulty-weighted efficiency of the agent's interventions, aggregated
over tasks and environments. The result depends only on:
\begin{enumerate}[label=(\roman*)]
  \item the computational structure of the environments (via $\Env$ and $K(\mu)$),
  \item the agent's performance (via $C_\pi$), and
  \item the choice of resource bias $\rho$.
\end{enumerate}
\end{remark}

\subsection{Learning efficiency}\label{subsec:learning}

Consider an oracle agent $\pi^*$ that knows $\mu$ completely and always
achieves $\IC_\mu^\rho(s, s')$ for every transition task it is given. This
agent has maximal Dimension~1 competence.
A resource-limited agent $\pi$ must learn about $\mu$ through experience.
We model this as a sequence of {tasks} where at each time step $t$, the
agent is presented with a transition task $(s_t, s_t^*)$.
The agent produces
a program $p_t$ based on its accumulated experience, and this program
outputs an action sequence that (if successful) transitions the
environment from $s_t$ to $s_t^*$. The $\rho$-cost of this intervention
is $\rho(p_t, \mu, s_t) \geq \IC_\mu^\rho(s_t, s_t^*)$.

In the course of executing $p_t$, the agent observes a sequence of
intermediate states, each revealing information about $T_\mu$. These
observations accumulate into the agent's experience:
\[
  H_t = H_{t-1} \;\cup\; \{\text{state transitions observed while
  executing task } t\}.
\]
This experience informs the agent's interventions on subsequent tasks.

\begin{definition}[Instantaneous regret]
The \emph{instantaneous regret} at task $t$ is
\[
  \Delta_t(\pi, \mu) \;=\; \rho(p_t, \mu, s_t) - \IC_\mu^\rho(s_t, s_t^*),
\]
measuring the excess cost of $\pi$'s intervention relative to the oracle
on this specific task.
\end{definition}

\begin{definition}[Cumulative regret]
The \emph{cumulative regret} of $\pi$ in $\mu$ up to task $T$ is
\[
  \mathrm{Regret}_T(\pi, \mu) \;=\; \sum_{t=1}^{T} \Delta_t(\pi, \mu).
\]
\end{definition}

The asymptotic growth rate of $\mathrm{Regret}_T$ captures learning
efficiency: $O(1)$ indicates eventual perfect learning; $O(\log T)$ is the
classical good-learner rate; $O(\sqrt{T})$ is typical of adversarial
settings; and $O(T)$ indicates failure to learn. Note that the regret is
measured relative to $\IC$ not relative to the agent's previous performance. 

The regret framework requires a specification of how transition tasks
$(s, s^*)$ are presented to the agent. We identify three natural choices,
each capturing a different aspect of learning.
\begin{enumerate}
\item \label{eval-A}\emph{Self-directed exploration.}
The agent selects its own targets. Regret is measured against whatever
transitions the agent attempts. This rewards agents that are strategically
curious and choose transitions that maximise information gain about the
environment's structure. Under this evaluation, agents employing strategies
such as surprise minimisation~\cite{friston2010} or compression
progress~\cite{schmidhuber2010} may achieve low regret, as these heuristics
are plausibly efficient strategies for building causal models.
\item \label{eval-B}\emph{Adversarial task sequence.}
Targets are chosen by an adversary to maximise cumulative regret. This
measures worst-case robustness.
This connects
to adversarial online learning and gives a minimax characterisation of
learning efficiency.

\item \label{eval-C}\emph{Average over all tasks.}
After $T$ steps of experience, the agent's generalisation is evaluated
by averaging the regret over all reachable transitions:
\[
  G_T(\pi, \mu) \;=\; \frac{1}{|S_\mu|^2}
  \sum_{s, s' \in S_\mu}
  \left[\rho(p_\pi(s, s', H_T), \mu, s) - \IC_\mu^\rho(s, s')\right],
\]
where $p_\pi(s, s', H_T)$ is the intervention $\pi$ produces for the
$(s, s')$ transition given experience $H_T$. This most directly measures
the quality of the agent's learned causal model, since it tests
generalisation to transitions the agent may not have encountered.
\end{enumerate}

These three evaluations are complementary rather than competing. 
Evaluation \ref{eval-A} measures exploration intelligence; Evaluation \ref{eval-B} measures robustness; and Evaluation \ref{eval-C} measures model quality.

\subsection{Dimension~1 properties do not constrain Dimension~2}

An important structural observation is that the properties of
Section~\ref{sec:axioms}, particularly minimality
(Property~\ref{ax:min}), apply to Dimension~1 and have no analogue in
Dimension~2. Minimality constrains the {reward functional} so that 
only the optimal intervention matters.
Dimension~2 evaluates a {learning process}, which inherently depends on
the full structure of the intervention set and the agent's interaction
history.

This separation has a consequence for existing theories.
Schmidhuber's compression progress and Friston's free energy minimisation
both violate minimality when proposed as {definitions of the goal}
of intelligence (Dimension~1), because they depend on the agent's learning
trajectory rather than on the optimal intervention. However, they may be
excellent {strategies for learning} (Dimension~2) that
help an agent close the gap to $\IC$ efficiently. The framework 
places these theories as accounts of the engine of learning rather than
definitions of the target of intelligence.

\subsection{The universal learning efficiency measure}

Aggregating across environments with the Solomonoff prior, define the
\emph{$\rho$-learning efficiency} of agent $\pi$ under evaluation
scheme $X \in \{A, B, C\}$ as:
\begin{equation}\label{eq:learning-eff}
  L_\rho^X(\pi) \;=\; \sum_{\mu} 2^{-K(\mu)} \cdot \ell^X(\pi, \mu),
\end{equation}
where $\ell^X(\pi, \mu)$ is a summary of $\pi$'s regret curve in $\mu$
under evaluation $X$ (e.g., the negative discounted cumulative regret
$-\sum_{t=1}^{\infty} \gamma^t \Delta_t$, or the negative asymptotic
regret rate).

The weighting by $2^{-K(\mu)}$ means that simple environments contribute
more to the learning efficiency score. This provides a principled escape
from the No Free Lunch theorems~\cite{wolpert1997}: while no learner can
outperform all others across {all} possible environments, a learner
biased toward simplicity (like a Solomonoff predictor) achieves low regret
in the environments that are weighted most heavily. The same inductive bias
that makes Dimension~1 tractable also makes Dimension~2 tractable.

\subsection{The complete two-dimensional characterisation}

The full intelligence of agent $\pi$ is characterised by the pair
\[
  \mathcal{I}_\rho(\pi) \;=\;
  (\Gamma_\rho(\pi, \cdot),\; L_\rho^X(\pi)),
\]
where $\Gamma_\rho(\pi, \cdot)$ is the competence curve
(Section~\ref{subsec:competence}), providing a full profile of the agent's
capabilities as a function of task difficulty, and $L_\rho^X$ measures
how quickly this competence improves with experience. When a scalar
summary of Dimension~1 is needed, the Solomonoff-weighted integral
$\Gamma_\rho(\pi)$ (Definition~\ref{def:scalar-comp}) serves as a
principled compression.

For the oracle, $L_\rho^X$ is maximal (zero regret from the start) and
the competence curve achieves its maximum at every difficulty level. For
real agents, there is a natural tradeoff: a simpler agent may learn quickly
(high $L_\rho^X$) but plateau with a low competence curve,
while a more complex agent may learn slowly but achieve higher competence
eventually.

\section{Separation of Resource Biases}\label{sec:separation}

We distinguish two access regimes for intervention programs. In the
\emph{oracle} regime, programs may query $T_\mu$ during execution; this
models a reasoning agent that already has a causal model of the
environment. In the \emph{bare} regime, programs receive no such access;
the transition table must be encoded within the program. The bare
regime models an agent that must contain its world model within itself.

\begin{theorem}[Action-count IC is polynomial-time computable]
\label{thm:ac-compute}
For finite $\mu$ with $|S_\mu| = n$ and $|A_\mu| = m$, and any
$s, s' \in S_\mu$, $\IC_\mu^{\rho_a}(s, s')$ is computable in time
$O(nm)$ via breadth-first search on the state graph. The result is
identical for oracle and bare regimes.
\end{theorem}

In the action-count case the cost depends only on output length, which is
the same whether the program queries the oracle or computes from a
hardcoded table.

\begin{proposition}[Program-length IC with oracle is bounded]
\label{prop:oracle-bounded}
For finite $\mu$ with $|S_\mu| = n$,
$\IC_\mu^{\rho_\ell, \mathrm{orc}}(s, s') \leq \lceil \log_2 n \rceil + O(1)$
for any reachable $(s, s')$. A fixed-length wrapper performs BFS using
the oracle, encoding only the target state in $\lceil \log_2 n \rceil$
bits.
\end{proposition}

The oracle regime under program-length bias is also tractable as with a
causal model in hand, planning costs are dominated by the
$O(\log n)$ bits needed to specify the target state. 

\begin{theorem}[Bare program-length IC is uncomputable]
\label{thm:bare-uncomp}
There is no algorithm that, given any finite environment $\mu$ and
states $s, s' \in S_\mu$, computes
$\IC_\mu^{\rho_\ell,\mathrm{bare}}(s,s')$.
\end{theorem}

\begin{corollary}[Knowledge cost]\label{cor:knowledge}
The gap $\mathcal{K}_\mu(s,s') = \IC_\mu^{\rho_\ell,\mathrm{bare}}(s,s')
- \IC_\mu^{\rho_\ell,\mathrm{orc}}(s,s')$ is non-negative for all
$\mu, s, s'$, bounded above by $K(\mu) + O(1)$, and unbounded across
environments. It quantifies the information content of learning $T_\mu$
to plan the $s \to s'$ transition.
\end{corollary}

The separation theorem has significance as the choice of
resource bias determines the computational status of intelligence.
Action-count bias makes intelligence algorithmically
tractable, with oracle access providing no advantage. Program-length
bias with oracle access makes planning cheap once the world model is
acquired, dominated by target-specification cost. Program-length bias
without oracle access makes intelligence a fundamentally uncomputable
target that no agent can provably reach. The bare/oracle gap formalises
what we mean by Dimension~2 content under program-length bias:
$\mathcal{K}_\mu$ bits of environmental knowledge that must be acquired
to plan effectively. Under action-count bias, this gap is identically
zero, so the Dimension~1/Dimension~2 distinction collapses.

\begin{remark}[Interpretation]
The knowledge cost $\mathcal{K}_\mu^{\rho_\ell}(s,s')$ measures the number
of bits of environmental knowledge an agent needs to plan the $(s,s')$
transition. It is bounded above by the total information content of the
environment ($K(\mu)$), but may be much smaller for specific transitions
that depend on only a fragment of the transition structure.

Under action-count bias, the knowledge cost is zero. This means that if an agent can learn the environment sufficiently accurately, then it can behave optimally in that environment. Learning is reduced to learning the environment.
Under program-length bias, optimal behaviour is uncomputable. An agent under such a bias must learn both the environment, and develop heuristics to shorten the programs they produce.
\end{remark}

\section{Discussion}\label{sec:discussion}

\subsection{Implications for pre-training}
The separation theorems clarify why pre-training-then-prompting works.
Pre-training is the Dimension~2 process of acquiring $T_\mu$, paying
the bare-regime cost up front. Prompting then operates in the oracle
regime, paying only $O(\log|S|)$ per task. The framework predicts that
prompt engineering should be the bottleneck once a sufficiently
general model is pre-trained, which matches current LLM practice.

Under action-count bias, learning the transition function $T_\mu$ (at most
$|S| \times |A|$ single-step observations) suffices to compute all IC values
exactly in polynomial time. The pre-training pipeline is fully computable.
Under program-length bias, the same observations give the agent complete
knowledge of $T_\mu$, reducing its situation to the oracle case. 
A bottom-up curriculum that learns single-action transitions first, then
composes actions is information-theoretically well-motivated for both biases.
Single-step transitions determine $T_\mu$ completely, from which all
multi-step IC values are either computable (action-count) or approximable
from above (program-length).

\subsection{Implications for superintelligence}

The competence curve has a ceiling, the oracle curve, which is fixed
by the environment. Self-improvement is therefore logistic rather than
exponential: an agent can approach the oracle but cannot exceed it.
The framework thus identifies three ceilings on intelligence. The
\emph{oracle ceiling} is set by the environment's structure
$\IC_\mu^\rho$. The \emph{computational ceiling} is set by complexity
classes: NP-hard environments cannot be solved in polynomial total
steps unless $\mathbf{P} = \mathbf{NP}$. The \emph{environment ceiling}
arises from Solomonoff weighting, where simple environments dominate
the scalar measure.

Bostrom~\cite{bostrom2014} (Ch.~4, pp.~62--65) models the trajectory of
machine intelligence as $dI/dt = \text{optimization power} /
\text{recalcitrance}$, and considers fast, moderate, and slow takeoff
scenarios depending on the shape of the resulting curve. 
Self-improvement means the agent modifies itself to increase its
competence curve. At each step, the improved agent has a higher
$\Gamma_\rho(\pi, k)$ for some range of $k$. But the oracle curve is
fixed as it is a property of the environment, not the agent. As the
agent's curve approaches the oracle, each increment of improvement
requires closing a smaller gap, yielding diminishing returns. The
trajectory is therefore logistic with potentially rapid growth when the agent
is far from the oracle (large gap, easy gains), decelerating as it
approaches the ceiling (small gap, hard gains), and asymptotically
bounded. 

The framework suggests, though does not prove, that collective human
competence may be quite far along the logistic curve. Human civilisation for simple environments (planetary motion, basic chemistry) is highly accurate and our competence is close to oracle-level. But these environments have high Solomonoff weight and therefore dominate the scalar competence measure.
This means the trajectory from human-level to superintelligence will be less
dramatic than the fast-takeoff scenarios of Bostrom~\cite{bostrom2014}, because
improvements are concentrated in complex environments with lower Solomonoff
weight. 

Even where superintelligence is bounded by a ceiling, it might still be
transformative if reaching that ceiling matters. The framework suggests
two cases where SI is required.
The first is where a problem cannot be specified without superintelligence. We cannot
formulate a question whose answer would be a breakthrough without already
having the conceptual apparatus that superintelligence would provide. Foundational
physics beyond the standard model and mathematics not yet conceived may
fall here. By construction these are unknown unknowns and the framework
cannot say much about them.
The second is where a problem can be specified but only a superintelligence can solve
it. If a problem is specifiable,
its IC is well-defined, and we can ask whether a narrow agent targeted
at the problem could achieve oracle-level performance. A narrow agent
has lower $K(\pi)$ than a superintelligence but matches its competence on the target
task. Under any resource bias preferring shorter programs, the narrow
agent has higher competence on that task than the superintelligence. This is supported
by current AI, where narrow systems such as AlphaFold and AlphaZero
outperform general approaches on their target domains.
The value of superintelligence then, is concentrated in unknown unknowns
and in problems requiring cross-domain reasoning that resists
narrow decomposition. Whether these residual cases dominate the value of
intelligence is an empirical question that current AI development has
not yet tested. A fuller analysis is left for future work.

\subsection{Recommended learning target}

For pre-training a universal agent, the framework suggests maximising
the competence curve under the combined bias
$\rho_{\mathrm{comb}} = \alpha|p| + \beta|\mathrm{output}|$. This
decomposes into two sub-objectives. First, learn the transition function
$T_\mu$ from observed transitions (Dimension~2). Single-action
transitions reveal individual entries of $T_\mu$; multi-step
transitions test the model's compositional structure. Pre-training on
diverse environments weighted by simplicity provides the inductive
bias. Second, given a learned model, find compressed planning
strategies (Dimension~1). The combined bias rewards agents that
discover compact representations of the environment's structure rather
than relying on lookup or exhaustive search. This is the implicit
target of large-scale pre-training, providing theoretical justification
for the pre-train-then-prompt paradigm.

\subsection{Limitations and future work}

The framework is restricted to deterministic environments; the
stochastic extension introduces additional considerations around
expected versus worst-case intervention cost. The competence-curve
formalism has not been empirically validated; concrete computation of
the curve for benchmark agents in finite environments would test the
framework's predictions. Connections to bounded
optimality~\cite{russell1995} suggest a formal analogy that we have not
fully developed, particularly the question of whether the competence
curve characterises an agent's bounded-optimal frontier as resources
scale. The implications of intervention complexity for AI alignment, and
in particular the question of what kinds of reward hacking it does and
does not eliminate, deserves a separate treatment that we leave for
future work. The conclusions regarding asymptotic limits of superintelligence depend directly on the Solomonoff prior. Whilst this prior is well motivated, a different prior would yield different conclusions.

\bibliographystyle{plain}
\bibliography{intervention_complexity}

\appendix

\section{Independence of the properties}\label{app:independence}

\subsection{Properties of IC}
\label{sec:axioms}
I claim that intervention complexity is an appropriate canonical, or universal, reward function. Following is a list of properties of IC, why the property is useful, and examples of rewards from the literature that do not satisfy the property. 

\begin{axiom}[Environment-derived]\label{ax:env}
$\IC_\mu^\rho(s, s')$ is computable from $\mu$, $s$, and $s'$ alone. No external
signal, preference data, or normative input is required.
\end{axiom}

This property comes directly from the definition of IC, and is required for any canonical reward function.
Standard reinforcement learning reward functions and RLHF reward
models~\cite{christiano2017} violate this property. When a designer
hand-crafts a reward (e.g., $+1$ for reaching a goal, $-0.1$ per timestep),
or when a reward model is learned from human preference data, the reward
depends on external normative input rather than on the environment's
computational structure alone. Two identical environments can be assigned
entirely different rewards depending on the designer's intent and in such cases the definition of intelligence depends on the designer.

A subtler example is Friston's free energy
principle~\cite{friston2010}: surprise-based reward depends on the
{agent's} epistemic state (how surprised a particular agent is),
not on the environment's structure alone. Two different agents in the
same environment have different surprise levels, so the reward is not
computable from $\mu$, $s$, and $s'$ alone.

\begin{axiom}[Universal]\label{ax:univ}
$IC$ is defined for all computable environments $\mu \in \Env$ and all
states $s, s' \in S_\mu$.
\end{axiom}

Once $\rho$ is defined, $IC$ is applicable to any environment.
Potential-based reward shaping~\cite{ng1999} defines
$R(s,s') = \gamma\Phi(s') - \Phi(s)$ for a potential function $\Phi$.
This is environment-derived for any {given} environment (once $\Phi$
is fixed), but there is no general computable procedure assigning a
meaningful potential function to every computable environment---the choice
of $\Phi$ requires domain-specific knowledge about which states are
desirable. Similarly, rewards based on domain-specific features (e.g.,
Euclidean distance to goal in spatial environments) are only definable
in environments possessing the relevant structure.

\begin{axiom}[Minimality]\label{ax:min}
The intervention complexity depends on the intervention set
$I_\mu(s,s')$ and the resource function $\rho$ only through the
minimum: only the cheapest intervention contributes, and all
non-optimal interventions are discarded. Formally, for any two triples
$(\mu_1, s_1, s_1')$ and $(\mu_2, s_2, s_2')$ with
$\min_{p \in I_{\mu_1}(s_1,s_1')} \rho(p, \mu_1, s_1) =
 \min_{p \in I_{\mu_2}(s_2,s_2')} \rho(p, \mu_2, s_2)$,
we have $\IC_{\mu_1}^\rho(s_1, s_1') = \IC_{\mu_2}^\rho(s_2, s_2')$.
\end{axiom}

Consider $R(\mu, s, s') = -H(I_\mu(s,s'))$, where $H$ measures the
richness of the intervention set---for instance, the number of distinct
programs below some length threshold that achieve the transition. This
rewards transitions that are achievable in {many} ways, regardless of
whether the best way is efficient. It measures accessibility or
robustness rather than intelligence.

More substantively, Schmidhuber's compression
progress~\cite{schmidhuber2010} defines reward as the rate of improvement
of the agent's world model. This depends on the agent's entire learning
history, not merely on the
optimal intervention. Two agents in the same environment, facing the same
transition, receive different rewards depending on what they have already
learned.

\begin{axiom}[Sensitivity]\label{ax:sens}
The intervention complexity distinguishes all distinct minimum costs.
That is, $\IC$ does not collapse distinct difficulties into a single
value.
\end{axiom}

Sparse binary rewards---the standard in many RL benchmarks---assign
$R=1$ if $s' = s^*_{\text{goal}}$ and $R=0$ otherwise. This collapses
all non-goal states to the same reward regardless of whether they are one
step or a billion steps from the goal. More generally, any threshold-based
reward ($R=1$ if $\IC \leq t$, $R=0$ otherwise) partitions transitions
into ``feasible'' and ``infeasible'' and discards all fine-grained ordering
within each class. The constant reward $R(\mu,s,s') = 1$ for all
transitions---a common ``survival bonus'' in RL---is the extreme case,
assigning the same value to every transition regardless of difficulty.
Sensitivity says that a proper measure of intelligence must
be more discriminating than pass/fail.

\begin{axiom}[Achievement preference]\label{ax:eff}
Higher intervention complexity corresponds to greater achievement: a
transition that requires more resources at the optimum is harder, and
$\IC$ reflects this directly. 
\end{axiom}

Define $R(\mu, s, s') = -\IC_\mu^\rho(s,s')$---reward that {decreases}
with intervention complexity, valuing ease over achievement. This
satisfies all other properties: it is environment-derived, universal,
depends only on IC (minimality), and distinguishes
all distinct IC values (sensitivity). But it reverses the direction:
easier transitions are rewarded more highly. This is a ``lazy'' reward
that assigns maximum value to doing nothing ($\IC = 0$) and penalises
any genuine accomplishment. Achievement preference
encodes the substantive judgement that intelligence is demonstrated by
solving {harder} problems, not easier ones.
\begin{remark}
$\IC_\mu^\rho$ specifies the best any agent can do---the efficiency with which
an agent achieves a transition---how close its actual cost is to the
IC minimum---is captured by Dimension~2 (Section~\ref{app:dim2}), not
by the reward functional.
\end{remark}

There are two consequences of the above properties. The first is a direct consequence of the
non-triviality of $\rho$, there exist $\mu, s, s_1', s_2'$ with
$\IC_\mu^\rho(s, s_1') \neq \IC_\mu^\rho(s, s_2')$ and hence:
\begin{lemma}[Non-degeneracy]
$\IC_\mu^\rho$ is non-degenerate:
there exist $\mu \in \Env$ and $s, s_1', s_2' \in S_\mu$
such that $\IC_\mu^\rho(s, s_1') \neq \IC_\mu^\rho(s, s_2')$.
\end{lemma}

\begin{lemma}[Ordinal invariance]
$\IC_\mu^\rho$
is ordinally invariant under change of UTM. Specifically, 
the ranking induced by $\IC$ is robust to the choice of UTM, with the
strength of the guarantee depending on the resource bias:
\begin{enumerate}[label=(\roman*)]
  \item For intrinsic biases whose cost depends only on the output
    (action count, energy): $\IC$ does not depend on the UTM at all, and
    the ranking is exactly preserved.
  \item For program-length bias: the invariance theorem for Kolmogorov
    complexity~\cite{lv2008} gives an additive bound
    $|\IC_{\mu,\UTM_1}^{\rho_\ell} - \IC_{\mu,\UTM_2}^{\rho_\ell}| \leq c$
    for a constant $c$ depending only on the UTM pair. The ranking is
    preserved for all pairs of transitions whose IC values differ by more
    than $2c$.
  \item For execution-time bias: the simulation theorem gives a polynomial
    bound $\IC_{\mu,\UTM_2}^{\rho_t} \leq \mathrm{poly}(\IC_{\mu,\UTM_1}^{\rho_t})$.
    The ranking is preserved for all pairs of transitions whose IC values
    differ by more than a polynomial factor.
\end{enumerate}
\end{lemma}

\begin{proof}
In each case, the argument is the same: if two transitions have IC values
that are sufficiently separated under $\UTM_1$ (where ``sufficiently''
depends on the invariance class of $\rho$), they remain separated in the
same direction under $\UTM_2$.

For (i): if $\rho$ depends only on the output action sequence (not on
the program or UTM), then $\IC_\mu^\rho$ is independent of the UTM
entirely, and invariance is exact.

For (ii): if $\IC_{\mu,\UTM_1}^{\rho_\ell}(s,s_1') <
\IC_{\mu,\UTM_1}^{\rho_\ell}(s,s_2')$ with a gap exceeding $2c$, then
under $\UTM_2$ each value shifts by at most $c$, so the ordering is
preserved.

For (iii): if $\IC_{\mu,\UTM_1}^{\rho_t}(s,s_1')$ and
$\IC_{\mu,\UTM_1}^{\rho_t}(s,s_2')$ differ by more than a polynomial
factor, the polynomial simulation overhead cannot reverse the ordering.
\end{proof}


\section{Metric structure of IC}\label{app:metric}


We examine how IC behaves under composition of transitions, showing
that it is sub-additive but not additive, and that it forms a directed
quasimetric on the state space.

\begin{proposition}[Sub-additivity]
For all $\mu, s, s', s''$ with $s'$ reachable from $s$ and $s''$
reachable from $s'$,
\[
  \IC_\mu^\rho(s, s'') \;\leq\; \IC_\mu^\rho(s, s') + \IC_\mu^\rho(s', s'') + O(1),
\]
where the $O(1)$ constant accounts for the overhead of composing two
programs (and is zero for action-count bias).
\end{proposition}

\begin{proof}
Given an optimal program $p_1$ for $s \to s'$ and an optimal program
$p_2$ for $s' \to s''$, construct a program $p$ that runs $p_1$ to
obtain an action sequence from $s$ to $s'$, then runs $p_2$ to obtain an
action sequence from $s'$ to $s''$, and outputs the concatenation. Under
action-count bias, $\rho_a(p) = \rho_a(p_1) + \rho_a(p_2)$ exactly.
Under program-length bias, $|p| = |p_1| + |p_2| + O(1)$ where the
$O(1)$ is the fixed-length concatenation wrapper.
\end{proof}

\begin{proposition}[Non-additivity]
There is no composition operation $\oplus$ such that for all $\mu, s, s', s''$,
\[
  \IC_\mu^\rho(s, s'') = \IC_\mu^\rho(s, s') \oplus \IC_\mu^\rho(s', s'').
\]
\end{proposition}

\begin{proof}
Set $s'' = s$. Then $\IC_\mu^\rho(s, s) = 0$ (the empty action sequence
achieves the identity transition, with minimal resource cost), while
$\IC_\mu^\rho(s, s')$ and $\IC_\mu^\rho(s', s)$ may both be arbitrarily
large. No operation $\oplus$ satisfies $0 = c_1 \oplus c_2$ for
arbitrary positive $c_1, c_2$.
\end{proof}

\begin{remark}[Directed quasimetric]
IC satisfies three of the four metric axioms: non-negativity
($\IC \geq 0$), identity ($\IC(s,s) = 0$), and the triangle inequality
(Proposition~\ref{prop:subadd}). It violates symmetry:
$\IC(s,s') \neq \IC(s',s)$ in general, as the non-additivity proof
illustrates (the return trip may be much harder than the outward
journey). IC is therefore a {directed quasimetric} on the state
space---a natural structure for environments with irreversible
transitions.

The combination of sub-additivity and non-additivity 
means the direct route from $s$ to $s''$ can be strictly cheaper than
any decomposition through intermediate states (non-additivity), but it
can never be more expensive (sub-additivity). An agent that discovers
direct shortcuts is more intelligent than one that routes through
intermediate states, and the $\IC$ correctly credits this.
\end{remark}


\section{Two dimensions in detail}\label{app:dim2}


Agent competence $\Gamma_\rho$ (Section~\ref{app:dim2}) provides
{Dimension~1}: a static measure of how well a fixed agent performs
relative to the oracle across all tasks and environments. The oracle
achieves the maximum possible competence curve; any real agent's curve
lies below it.

But agent competence is a snapshot as it evaluates the agent at a single
point in time. A real agent does not have perfect knowledge. It must
build a model of $\mu$ through interaction, and its competence should
improve as it gains experience. The rate at which the competence curve
improves is {Dimension~2: learning efficiency}. In this section
we formalise Dimension~2 and discuss its relationship to Dimension~1.

\subsection{The oracle--learner decomposition}

Consider an oracle agent $\pi^*$ that knows $\mu$ completely and always
achieves $\IC_\mu^\rho(s, s')$ for every transition task it is given. This
agent has maximal Dimension~1 competence: its per-task contribution is
$\IC^2/\IC = \IC$, giving full credit for each task's difficulty.

A resource-limited agent $\pi$ must learn about $\mu$ through experience.
We model this as a sequence of {tasks}: at each time step $t$, the
agent is presented with a transition task $(s_t, s_t^*)$---a start state
and a target state, which may be arbitrarily far apart. The agent produces
a program $p_t$ based on its accumulated experience, and this program
outputs an action sequence that (if successful) transitions the
environment from $s_t$ to $s_t^*$. The $\rho$-cost of this intervention
is $\rho(p_t, \mu, s_t) \geq \IC_\mu^\rho(s_t, s_t^*)$.

In the course of executing $p_t$, the agent observes a sequence of
intermediate states, each revealing information about $T_\mu$. These
observations accumulate into the agent's experience:
\[
  H_t = H_{t-1} \;\cup\; \{\text{state transitions observed while
  executing task } t\}.
\]
This experience informs the agent's interventions on subsequent tasks.

\begin{definition}[Instantaneous regret]
The \emph{instantaneous regret} at task $t$ is
\[
  \Delta_t(\pi, \mu) \;=\; \rho(p_t, \mu, s_t) - \IC_\mu^\rho(s_t, s_t^*),
\]
measuring the excess cost of $\pi$'s intervention relative to the oracle
on this specific task.
\end{definition}

\begin{definition}[Cumulative regret]
The \emph{cumulative regret} of $\pi$ in $\mu$ up to task $T$ is
\[
  \mathrm{Regret}_T(\pi, \mu) \;=\; \sum_{t=1}^{T} \Delta_t(\pi, \mu).
\]
\end{definition}

The asymptotic growth rate of $\mathrm{Regret}_T$ captures learning
efficiency: $O(1)$ indicates eventual perfect learning; $O(\log T)$ is the
classical good-learner rate; $O(\sqrt{T})$ is typical of adversarial
settings; and $O(T)$ indicates failure to learn. Note that the regret is
measured relative to $\IC$---the optimal intervention cost for each
task---not relative to the agent's previous performance. An agent that
solves an easy task inefficiently incurs regret even if the task is trivial,
while an agent that solves a hard task optimally incurs none.

\subsection{Three complementary evaluations}\label{sec:three-evals}

The regret framework requires a specification of how transition tasks
$(s, s^*)$ are presented to the agent. We identify three natural choices,
each capturing a different aspect of learning.

\medskip\noindent\textbf{(A) Self-directed exploration.}
The agent selects its own targets. Regret is measured against whatever
transitions the agent attempts. This rewards agents that are strategically
curious---choosing transitions that maximise information gain about the
environment's structure. Under this evaluation, agents employing strategies
such as surprise minimisation~\cite{friston2010} or compression
progress~\cite{schmidhuber2010} may achieve low regret, as these heuristics
are plausibly efficient strategies for building causal models.

\medskip\noindent\textbf{(B) Adversarial task sequence.}
Targets are chosen by an adversary to maximise cumulative regret. This
measures worst-case robustness: how well does the agent perform when the
sequence of tasks is specifically designed to be difficult? This connects
to adversarial online learning and gives a minimax characterisation of
learning efficiency.

\medskip\noindent\textbf{(C) Average over all tasks.}
After $T$ steps of experience, the agent's generalisation is evaluated
by averaging the regret over all reachable transitions:
\[
  G_T(\pi, \mu) \;=\; \frac{1}{|S_\mu|^2}
  \sum_{s, s' \in S_\mu}
  \left[\rho(p_\pi(s, s', H_T), \mu, s) - \IC_\mu^\rho(s, s')\right],
\]
where $p_\pi(s, s', H_T)$ is the intervention $\pi$ produces for the
$(s, s')$ transition given experience $H_T$. This most directly measures
the quality of the agent's learned causal model, since it tests
generalisation to transitions the agent may not have encountered.

\medskip

These three evaluations are complementary rather than competing. Evaluation
(C) measures {model quality}---how accurate is the agent's causal model
at time $T$? Evaluation (A) measures {exploration intelligence}---how
cleverly does the agent gather information? Evaluation (B) measures
{robustness}---how reliably does the agent learn regardless of task
ordering?

\subsection{Dimension~1 properties do not constrain Dimension~2}

An important structural observation is that the properties of
Section~\ref{sec:axioms}---particularly minimality
(Property~\ref{ax:min})---apply to Dimension~1 and have no analogue in
Dimension~2. Minimality constrains the {reward functional}: when
evaluating a single state transition, only the optimal intervention matters.
Dimension~2 evaluates a {learning process}, which inherently depends on
the full structure of the intervention set and the agent's interaction
history.

This separation has a clarifying consequence for existing theories.
Schmidhuber's compression progress and Friston's free energy minimisation
both violate minimality when proposed as {definitions of the goal}
of intelligence (Dimension~1), because they depend on the agent's learning
trajectory rather than on the optimal intervention. However, they may be
excellent {strategies for learning} (Dimension~2)---heuristics that
help an agent close the gap to $\IC$ efficiently. The framework correctly
places these theories as accounts of the engine of learning rather than
definitions of the target of intelligence.

\subsection{The universal learning efficiency measure}

Aggregating across environments with the Solomonoff prior, define the
\emph{$\rho$-learning efficiency} of agent $\pi$ under evaluation
scheme $X \in \{A, B, C\}$ as:
\begin{equation}
  L_\rho^X(\pi) \;=\; \sum_{\mu} 2^{-K(\mu)} \cdot \ell^X(\pi, \mu),
\end{equation}
where $\ell^X(\pi, \mu)$ is a summary of $\pi$'s regret curve in $\mu$
under evaluation $X$ (e.g., the negative discounted cumulative regret
$-\sum_{t=1}^{\infty} \gamma^t \Delta_t$, or the negative asymptotic
regret rate).

The weighting by $2^{-K(\mu)}$ means that simple environments contribute
more to the learning efficiency score. This provides a principled escape
from the No Free Lunch theorems~\cite{wolpert1997}: while no learner can
outperform all others across {all} possible environments, a learner
biased toward simplicity (like a Solomonoff predictor) achieves low regret
in the environments that are weighted most heavily. The same inductive bias
that makes Dimension~1 tractable also makes Dimension~2 tractable.

\subsection{The complete two-dimensional characterisation}

The full intelligence of agent $\pi$ is characterised by the pair
\[
  \mathcal{I}_\rho(\pi) \;=\;
  (\Gamma_\rho(\pi, \cdot),\; L_\rho^X(\pi)),
\]
where $\Gamma_\rho(\pi, \cdot)$ is the competence curve
(Section~\ref{app:dim2}), providing a full profile of the agent's
capabilities as a function of task difficulty, and $L_\rho^X$ measures
how quickly this competence improves with experience. When a scalar
summary of Dimension~1 is needed, the Solomonoff-weighted integral
$\Gamma_\rho(\pi)$ (Definition~\ref{app:dim2}) serves as a
principled compression.

For the oracle, $L_\rho^X$ is maximal (zero regret from the start) and
the competence curve achieves its maximum at every difficulty level. For
real agents, there is a natural tradeoff: a simpler agent may learn quickly
(high $L_\rho^X$) but plateau with a low competence curve,
while a more complex agent may learn slowly but achieve higher competence
eventually.

\begin{remark}[Possible recursive structure]
Under evaluation (A), the agent's choice of what to explore is itself an
intervention---an action taken to produce a desired epistemic state. The
efficiency of this meta-intervention could, in principle, be measured by
$\IC$ in a meta-environment where the state space is the agent's epistemic
state. If this construction is coherent, Dimension~2 partially collapses
into Dimension~1 applied at a meta-level, potentially unifying the two
dimensions. Avoiding circularity in this construction (the agent's epistemic
state depends on its model of $\IC$, which is what it is trying to learn)
requires careful treatment, and we leave this as an open question.
\end{remark}


\section{Full separation results}\label{app:proofs}


The choice of resource bias $\rho$ determines not only the {meaning}
of the intelligence measure but also its {computability}. In this
section we restrict to finite deterministic environments and establish a
sharp separation between action-count and program-length biases. The
analysis also reveals that the gap between oracle-access and bare
(no-oracle) IC quantifies the information content of Dimension~2.

\subsection{Oracle and bare intervention complexity}

Our framework (Section~\ref{sec:framework}) gives the intervening program
$p$ oracle access to the environment $\mu$: the program may query $T_\mu$
during its computation. This models an agent with complete knowledge of the
environment. We now distinguish this from the case where the program must
encode its own knowledge.

\begin{definition}[Oracle and bare IC]
Let $\mu$ be a finite deterministic environment and $s, s' \in S_\mu$.
\begin{enumerate}[label=(\roman*)]
\item The \emph{oracle IC} is defined as before:
\[
  \IC_\mu^{\rho,\mathrm{orc}}(s,s') = \min\{|p| : \UTM(p, \mu, s)
  \text{ outputs a valid action sequence from } s \text{ to } s'\}.
\]
\item The \emph{bare IC} is:
\[
  \IC_\mu^{\rho,\mathrm{bare}}(s,s') = \min\{|p| : \UTM(p)
  \text{ outputs a valid action sequence from } s \text{ to } s' \text{ in } \mu\},
\]
where $p$ receives no input and has no oracle access.
\end{enumerate}
\end{definition}

Since a bare program can be converted to an oracle program (that simply
ignores the oracle), we have
$\IC_\mu^{\rho,\mathrm{orc}}(s,s') \leq \IC_\mu^{\rho,\mathrm{bare}}(s,s')$
for program-length bias. The gap between them measures the cost of
{environmental knowledge}.

\subsection{Action-count bias: oracle access is irrelevant}

Under action-count bias $\rho_a$, the cost of a program is the number of
actions in its output, not its description length. Oracle access helps
the program {find} the shortest action sequence, but does not change
the sequence's length.

\begin{proposition}\label{prop:ac-oracle}
For all finite deterministic environments $\mu$ and states $s, s' \in S_\mu$,
\[
  \IC_\mu^{\rho_a,\mathrm{orc}}(s,s')
  = \IC_\mu^{\rho_a,\mathrm{bare}}(s,s')
  = d_{G_\mu}(s,s'),
\]
where $d_{G_\mu}(s,s')$ is the shortest-path distance from $s$ to $s'$
in the directed graph $G_\mu$ induced by $T_\mu$ (with $d = \infty$ if
$s'$ is unreachable).
\end{proposition}

\begin{proof}
Under $\rho_a$, the cost of program $p$ is the length of its output action
sequence, regardless of $p$'s description length or oracle access. Any
program outputting a shortest action sequence from $s$ to $s'$ achieves
cost $d_{G_\mu}(s,s')$, whether or not it uses oracle access. Conversely,
no program can achieve cost less than $d_{G_\mu}(s,s')$, since the output
must be a valid path and no valid path is shorter than the shortest one.
\end{proof}

\begin{theorem}[Computability of action-count IC]\label{thm:ac-compute-app}
There exists an algorithm that, given a finite deterministic environment
$\mu = (S, A, T)$ and states $s, s' \in S$, computes
$\IC_\mu^{\rho_a}(s,s')$ in time $O(|S| \cdot |A|)$.
\end{theorem}

\begin{proof}
$\IC_\mu^{\rho_a}(s,s') = d_{G_\mu}(s,s')$, the shortest-path distance
in the directed graph $G_\mu$ with $|S|$ vertices and $|S| \cdot |A|$
edges. Breadth-first search from $s$ computes shortest-path distances
to all vertices in $O(|V| + |E|) = O(|S| + |S| \cdot |A|) = O(|S| \cdot |A|)$ time.
\end{proof}

\subsection{Program-length bias with oracle access: bounded and near-trivial}

\begin{proposition}\label{prop:oracle-bounded-app}
For any finite deterministic environment $\mu$ with $|S_\mu| = n$ and
all states $s, s' \in S_\mu$ with $s'$ reachable from $s$,
\[
  \IC_\mu^{\rho_\ell,\mathrm{orc}}(s,s') \;\leq\; \lceil \log_2 n \rceil + C_{\mathrm{BFS}},
\]
where $C_{\mathrm{BFS}}$ is a constant (independent of $\mu$, $s$, $s'$)
equal to the description length of a breadth-first search algorithm.
\end{proposition}

\begin{proof}
The program consists of a fixed BFS routine (which queries the oracle
$T_\mu$ to explore the graph, finds a shortest path from $s$ to $s'$,
and outputs the corresponding action sequence) together with a specification
of the target state $s'$, requiring $\lceil \log_2 n \rceil$ bits.
The BFS routine is a single fixed program whose length $C_{\mathrm{BFS}}$
does not depend on $\mu$.
\end{proof}

\begin{corollary}\label{cor:oracle-narrow}
For any finite environment with $|S_\mu| = n$, the oracle IC values for reachable pairs under program-length bias lie in an interval of width at most
$\lceil \log_2 n \rceil + O(1)$.
\end{corollary}

\begin{proof}
$\IC_\mu^{\rho_\ell,\mathrm{orc}}(s,s') \geq 0$ for all $s,s'$, and
by a counting argument, at most $2^k - 1$ programs have length less than
$k$, so at most $2^k - 1$ distinct transitions can have oracle IC less
than $k$. Since there are at most $n^2$ state pairs,
$\IC_\mu^{\rho_\ell,\mathrm{orc}}(s,s') \geq \lceil \log_2 n^2 \rceil
- O(1) = 2\lceil \log_2 n \rceil - O(1)$ for all reachable pairs. Combined with
the upper bound, oracle IC values are concentrated in a band of width
$O(\log n)$.
\end{proof}

\begin{remark}
This result shows that with oracle access, program-length IC carries
primarily {target-specification} information
rather than {structural} information about the
environment. This is in contrast to action-count IC, whose values
range from $0$ to the diameter of $G_\mu$ (potentially $|S| - 1$) and
encode information about the graph structure (distances, clusters,
bottlenecks).
\end{remark}

\subsection{Program-length bias without oracle access: uncomputable}

Without oracle access, the program must encode its own knowledge of
the environment's transition structure. This makes bare IC 
harder to compute.

\begin{theorem}[Uncomputability of bare program-length IC]\label{thm:bare-uncomp-app}
There is no algorithm that, given a finite deterministic environment
$\mu = (S, A, T)$ and states $s, s' \in S$, computes
$\IC_\mu^{\rho_\ell,\mathrm{bare}}(s,s')$.
\end{theorem}

\begin{proof}
We reduce the computation of Kolmogorov complexity $K(x)$ to the
computation of bare IC.

Let $x = x_1 x_2 \cdots x_n \in \{0,1\}^n$ be an arbitrary binary string.
Construct the environment $\mu_x = (S_x, A_x, T_x)$ as follows:
\begin{itemize}
  \item $S_x = \{s_0, s_1, \ldots, s_n, s_f, s_{\mathrm{fail}}\}$,
    with $|S_x| = n + 3$.
  \item $A_x = \{0, 1\}$.
  \item $T_x$ is defined by:
  \begin{align*}
    T_x(s_i, x_{i+1}) &= s_{i+1} &&\text{for } 0 \leq i < n, \\
    T_x(s_i, 1 - x_{i+1}) &= s_{\mathrm{fail}} &&\text{for } 0 \leq i < n, \\
    T_x(s_n, a) &= s_f &&\text{for both } a \in \{0,1\}, \\
    T_x(s_f, a) &= s_f &&\text{for both } a \in \{0,1\}, \\
    T_x(s_{\mathrm{fail}}, a) &= s_{\mathrm{fail}} &&\text{for both } a \in \{0,1\}.
  \end{align*}
\end{itemize}

The environment is a ``gated corridor'': at each step $i$, action $x_{i+1}$
advances the agent from $s_i$ to $s_{i+1}$, while the wrong action sends
the agent irreversibly to $s_{\mathrm{fail}}$. The only action sequence
that reaches $s_f$ from $s_0$ (other than paths through $s_f$ itself,
once reached) has its first $n$ actions equal to $x_1 x_2 \cdots x_n$,
followed by one additional action (either 0 or 1) to transition from
$s_n$ to $s_f$.

\medskip\noindent\textbf{Claim:}
$\IC_{\mu_x}^{\rho_\ell,\mathrm{bare}}(s_0, s_f) = K(x) + O(1)$.

\medskip\noindent\textbf{Upper bound.}
Given any program $q$ of length $K(x)$ that outputs $x$, we construct
a program $p$ that: (a)~runs $q$ to obtain $x = x_1 \cdots x_n$;
(b)~outputs the action sequence $x_1, x_2, \ldots, x_n, 0$.
This is a valid path from $s_0$ to $s_f$ in $\mu_x$. The program $p$
consists of $q$ together with a fixed wrapper of constant size, so
$|p| = K(x) + O(1)$.

\medskip\noindent\textbf{Lower bound.}
Let $p$ be any program such that $\UTM(p)$ outputs a valid action sequence
from $s_0$ to $s_f$ in $\mu_x$. By the sub-additivity of Kolmogorov
complexity~\cite{lv2008}, $|p| \geq K(x)$.

\medskip\noindent\textbf{Conclusion of reduction.}
The map $x \mapsto \mu_x$ is computable (given $x$, the transition
table $T_x$ can be constructed in $O(n)$ time). If there existed an
algorithm $\mathcal{A}$ computing
$\IC_\mu^{\rho_\ell,\mathrm{bare}}(s, s')$ for all finite environments,
then for any $x \in \{0,1\}^*$ we could compute $K(x)$ to within an
additive constant via
\[
  K(x) = \mathcal{A}(\mu_x, s_0, s_f) + c,
\]
where $c$ is a fixed constant determined by the wrapper construction.
This would make $K$ computable up to additive constant, contradicting
the uncomputability of $K$~\cite{lv2008}.
\end{proof}

\begin{remark}
Bare IC remains upper-semi-computable: one can enumerate programs in order
of increasing length, execute each, and check whether the output is a valid
path. Each success yields an upper bound, and the sequence of bounds
converges to $\IC_\mu^{\rho_\ell,\mathrm{bare}}(s,s')$. The
uncomputability is the inability to know when the optimum has been reached.
\end{remark}

\subsection{The oracle--bare gap quantifies Dimension~2}

The separation between oracle and bare IC provides an
information-theoretic characterisation of Dimension~2 (learning
efficiency).

\begin{definition}[Knowledge cost]
The {knowledge cost} of a transition $(s, s')$ in environment $\mu$
under resource bias $\rho$ is:
\[
  \mathcal{K}_\mu^\rho(s,s') \;=\;
  \IC_\mu^{\rho,\mathrm{bare}}(s,s')
  - \IC_\mu^{\rho,\mathrm{orc}}(s,s').
\]
\end{definition}

\begin{proposition}\label{prop:knowledge-cost}
The knowledge cost satisfies:
\begin{enumerate}[label=(\roman*)]
  \item $\mathcal{K}_\mu^\rho(s,s') \geq 0$ for all $\mu, s, s'$ and any
    resource bias $\rho$.
  \item Under program-length bias,
    $\mathcal{K}_\mu^{\rho_\ell}(s,s') \leq K(\mu) + O(1)$, where $K(\mu)$
    is the Kolmogorov complexity of the environment's transition table.
  \item Under program-length bias, the knowledge cost is unbounded across
    environments: for any constant $c$, there exists an environment $\mu$
    and states $s, s'$ with $\mathcal{K}_\mu^{\rho_\ell}(s,s') > c$.
\end{enumerate}
\end{proposition}

\begin{proof}
(i) Any bare program is also a valid oracle program (ignoring the oracle),
so $\IC^{\rho,\mathrm{bare}} \geq \IC^{\rho,\mathrm{orc}}$ for any
resource bias.

(ii) Given an oracle program $p^*$ achieving
$\IC_\mu^{\rho_\ell,\mathrm{orc}}(s,s')$, we can construct a bare program
that first hardcodes the transition table $T_\mu$ (requiring
$K(\mu) + O(1)$ bits), then simulates $p^*$ with this hardcoded table
serving as the oracle. The total length is $|p^*| + K(\mu) + O(1)$,
giving $\IC_\mu^{\rho_\ell,\mathrm{bare}}(s,s') \leq
\IC_\mu^{\rho_\ell,\mathrm{orc}}(s,s') + K(\mu) + O(1)$.

(iii) Under program-length bias with oracle access, IC is bounded by
$\log|S| + O(1)$ (Proposition~\ref{prop:oracle-bounded-app}), while the bare
IC can be arbitrarily large in environments with high
Kolmogorov complexity. As $K(\mu)$ grows without bound across the class
of computable environments, the gap grows correspondingly.
\end{proof}

\begin{remark}[Interpretation]
The knowledge cost $\mathcal{K}_\mu^{\rho_\ell}(s,s')$ measures the number
of bits of environmental knowledge an agent needs to plan the $(s,s')$
transition. It is bounded above by the total information content of the
environment ($K(\mu)$), but may be much smaller for specific transitions
that depend on only a fragment of the transition structure.

Under action-count bias, the knowledge cost is zero. This means that if an agent can learn the environment sufficiently accurately, then it can behave optimally in that environment. Learning is reduced to learning the environment.

Under program-length bias, the situation is different. Optimal behaviour is uncomputable. An agent under such a bias must learn both the environment, and develop heuristics to shorten the programs they produce.
\end{remark}

\subsection{Summary of the separation}

Table~\ref{tab:separation} summarises the results. The choice of resource
bias determines not just the semantics of the intelligence measure but its
fundamental computational character.

\begin{table}[ht]
\centering
\small
\begin{tabular}{@{}lll@{}}
\toprule
  & \textbf{Action-count $\rho_a$}
  & \textbf{Program-length $\rho_\ell$} \\
\midrule
Oracle IC range
  & $[0,\; |S|-1]$
  & $[0,\; \lceil\log|S|\rceil + O(1)]$ \\
Oracle IC information
  & Rich (graph distances)
  & Narrow (target specification) \\
Oracle IC computability
  & Polynomial ($O(|S| \cdot |A|)$)
  & Upper-semi-computable \\
Bare IC range
  & Same as oracle
  & $[0, \;\infty)$ \\
Oracle--bare gap
  & Zero
  & Up to $K(\mu) + O(1)$; encodes Dimension~2 \\
Bare IC computability
  & Polynomial (same algorithm)
  & Not computable \\
\bottomrule
\end{tabular}
\caption{Separation of resource biases for finite deterministic environments.}
\label{tab:separation}
\end{table}

\begin{remark}[Implications for pre-training]
The separation has a direct consequence for pre-training universal agents.
Under action-count bias, learning the transition function $T_\mu$ (at most
$|S| \times |A|$ single-step observations) suffices to compute all IC values
exactly in polynomial time. The pre-training pipeline is fully computable.

Under program-length bias, the same observations give the agent complete
knowledge of $T_\mu$, reducing its situation to the oracle case. But the
remaining oracle IC, while bounded, is not computable.
Knowing the world (Dimension~2) is achievable by finite observation,
but optimally exploiting that knowledge (Dimension~1 under
program-length bias) requires finding compressed representations, which
inherits the uncomputability of Kolmogorov complexity.

A bottom-up curriculum that learns single-action transitions first, then
composes actions is information-theoretically well-motivated for both biases.
Single-step transitions determine $T_\mu$ completely, from which all
multi-step IC values are either computable (action-count) or approximable
from above (program-length).
\end{remark}


\section{Implications for superintelligence}\label{app:super}


The competence curve provides a formal framework for analysing
superintelligence claims. In this section we identify three ceilings on
intelligence, argue that the trajectory of improvement is logistic rather
than exponential, and offer a conjecture about where human civilisation
sits on this curve.

\subsection{Three ceilings}

\medskip\noindent\textbf{Ceiling 1: The oracle ceiling.}
For any environment $\mu$ and resource bias $\rho$, the oracle competence
$\IC_\mu^\rho(s, s')$ is a fixed property of the environment. No agent can
achieve a lower intervention cost. The competence curve of any agent is
bounded above by the oracle curve, which depends only on $\mu$ and $\rho$.
This ceiling is absolute.

\medskip\noindent\textbf{Ceiling 2: The computational ceiling.}
Under the unified action-space interpretation
(Section~\ref{app:proofs}), the oracle itself is constrained by
computational complexity. If the environment encodes NP-hard planning
problems, the oracle IC is superpolynomial in the natural problem
size. No computational agent, however intelligent, can solve such
environment classes in polynomial total steps unless
$\mathbf{P} = \mathbf{NP}$. An agent with
a higher competence curve than humans under the same resource
bias (quality superintelligence) has provable limits set by the complexity hierarchy.

\medskip\noindent\textbf{Ceiling 3: The environment ceiling.}
The Solomonoff weighting $2^{-K(\mu)}$ means that simple environments
contribute exponentially more to the scalar competence $\Gamma_\rho(\pi)$
than complex ones. An agent that excels only in complex environments
(high $K(\mu)$) contributes little to its overall score.
The environments that matter most are those with short descriptions and
these are precisely the environments for which human science has had the
longest to develop good models.

\subsection{Self-improvement is logistic, not exponential}

Bostrom~\cite{bostrom2014} (Ch.~4, pp.~62--65) models the trajectory of
machine intelligence as $dI/dt = \text{optimization power} /
\text{recalcitrance}$, and considers fast, moderate, and slow takeoff
scenarios depending on the shape of the resulting curve. Our framework
makes the ceiling explicit and the shape a mathematical consequence.

Self-improvement means the agent modifies itself to increase its
competence curve. At each step, the improved agent has a higher
$\Gamma_\rho(\pi, k)$ for some range of $k$. But the oracle curve is
fixed as it is a property of the environment, not the agent. As the
agent's curve approaches the oracle, each increment of improvement
requires closing a smaller gap, yielding diminishing returns. The
trajectory is therefore logistic with potentially rapid growth when the agent
is far from the oracle (large gap, easy gains), decelerating as it
approaches the ceiling (small gap, hard gains), and asymptotically
bounded.

For action-count bias, the oracle is computable in polynomial time
(Theorem~\ref{thm:ac-compute-app}), so the ceiling is reachable in
principle. Once the agent has learned the environment, it can achieve
oracle-level performance with BFS. There is no room for unbounded
improvement. For program-length bias, the oracle IC is uncomputable in
general (Theorem~\ref{thm:bare-uncomp-app}), so no agent can provably
reach the ceiling, but the ceiling still exists, and upper bounds on
IC are computable and improve over time. Self-improvement means finding
progressively shorter programs, with Kolmogorov complexity as the
unreachable limit.

\subsection{Where is human civilisation on the curve?}

The framework suggests, though does not prove, that collective human
competence may be further up the logistic curve than the public discourse
assumes.

For simple environments (low $K(\mu)$) such as the physics of planetary motion,
basic chemistry, and structural engineering, human civilisation has had
millennia to learn the transition functions. Our models of $T_\mu$ for
these environments are highly accurate, and our competence is
correspondingly close to oracle-level. Since these environments carry the
highest Solomonoff weight, they dominate the scalar competence measure.

For moderately complex environments such as protein folding, climate dynamics,
economic systems, we are probably significantly below oracle-level. The
Dimension~2 gap is large. These are the domains where current AI is
producing the most dramatic gains.

For highly complex environments such as social dynamics, consciousness,
open-ended mathematics, we are very likely far from oracle-level. But these
environments have low Solomonoff weight and contribute
correspondingly less to the scalar measure.

The Solomonoff weighting means the trajectory from human-level to superintelligence will be less
dramatic than the fast-takeoff scenarios of Bostrom~\cite{bostrom2014}
suggest, because
they are concentrated in complex environments with lower Solomonoff
weight. This does not mean that solving complicated problems or problems in complicated environments won't have a large impact, either for good or bad. 

\subsection{Is superintelligence necessary?}

Even where superintelligence is bounded by a ceiling, it might still be
transformative if reaching that ceiling matters. The framework suggests
two cases where SI is required.

The first is where a problem cannot be specified without superintelligence. We cannot
formulate a question whose answer would be a breakthrough without already
having the conceptual apparatus that superintelligence would provide. Foundational
physics beyond the standard model and mathematics not yet conceived may
fall here. By construction these are unknown unknowns and the framework
cannot say much about them.

The second is where a problem can be specified but only a superintelligence can solve
it. If a problem is specifiable,
its IC is well-defined, and we can ask whether a narrow agent targeted
at the problem could achieve oracle-level performance. A narrow agent
has lower $K(\pi)$ than a superintelligence but matches its competence on the target
task. Under any resource bias preferring shorter programs, the narrow
agent has higher competence on that task than the SI. This is supported
by current AI, where narrow systems such as AlphaFold and AlphaZero
outperform general approaches on their target domains.

The combination of these observations narrows the case for
transformative superintelligence. Superintelligence is bounded above by the oracle ceiling, dominated
by the Solomonoff weighting on simple environments where humans already
perform well, and arguably unnecessary for problems that can be
specified at all. Its remaining value is concentrated in unknown unknowns
and in problems requiring cross-domain reasoning that resists
narrow decomposition. Whether these residual cases dominate the value of
intelligence is an empirical question that current AI development has
not yet tested. A fuller analysis is left for future work.

\subsection{Recommended learning target}

For pre-training a universal agent, the framework suggests a clear
objective: maximise the competence curve $\Gamma_\rho(\pi, k)$
across environments, using the combined resource bias
$\rho_{\mathrm{comb}} = \alpha |p| + \beta |\mathrm{output}|$ of
Remark~\ref{app:proofs}.

In practice, this decomposes into two sub-objectives corresponding to
the two dimensions:
\begin{enumerate}[label=(\roman*)]
  \item \textbf{Learn the transition function $T_\mu$} (Dimension~2).
    Build a causal model of the environment from observed transitions.
    This is the bottom-up curriculum: single-action transitions first
    (revealing individual entries of $T_\mu$), then multi-step
    transitions (testing the model's ability to compose). Pre-training
    on diverse environments weighted by simplicity (the Solomonoff
    prior in practice) provides the inductive bias.
  \item \textbf{Find compressed planning strategies} (Dimension~1).
    Given the learned model, find short programs that produce efficient
    action sequences. The combined bias penalises both long programs
    (inelegant strategies) and long outputs (wasteful solutions),
    rewarding agents that discover compact representations of the
    environment's structure.
\end{enumerate}
This is, in essence, what large-scale pre-training already does.
learn a world model from diverse data (Dimension~2), then use it for
efficient planning and generation (Dimension~1). The framework provides
a theoretical justification for this approach and identifies the
combined bias $\rho_{\mathrm{comb}}$ as the principled choice of
resource function.


\section{Connection to large language models}\label{app:llm}

The pre-train-then-prompt paradigm of modern large language models (LLMs)
is an empirical instantiation of the Dimension~1\,/\,Dimension~2
separation.

Pre-training on large text corpora is the Dimension~2 learning phase: the
model observes billions of single-step transitions (next-token predictions
reflecting real-world causal processes, reasoning chains, and physical
regularities) and builds a compressed internal model of the generating
process. This corresponds to learning the transition function $T_\mu$ from
single-step observations and is the bottom-up curriculum identified in
Section~\ref{sec:separation} as information-theoretically well-motivated.
The training objective (next-token prediction loss) is a Dimension~2
strategy, an efficient heuristic for closing the regret gap, and consistent
with the framework's placement of compression progress and surprise
minimisation as theories of the learning engine rather than definitions of
the intelligence target (Section~\ref{app:dim2}).

Once pre-training is complete, the model possesses an approximate oracle which is
a compressed world model from which it can plan interventions. Prompting
is then goal specification and the
model's remaining task is to find an efficient output
(action sequence) that reaches the specified target (Dimension~1). Consistent with
Proposition~\ref{prop:oracle-bounded-app}, this is relatively cheap once the
model is built; the prompt (target specification) is small relative to
the pre-training investment (world-model acquisition).

Several empirical phenomena in LLM practice align with predictions of the
framework:
\begin{enumerate}[label=(\roman*)]
\item {Resistance to catastrophic forgetting.} Pre-trained LLMs can
  be fine-tuned for diverse tasks (coding, medical reasoning, legal
  analysis) without destroying general capabilities. The framework predicts
  this: pre-training acquires a goal-independent causal model ($T_\mu$)
  that is shared across all possible target states, so task-specific
  fine-tuning adds only a thin goal-specification layer rather than
  overwriting the general model.
\item {Prompt engineering as the primary bottleneck.} Once a model is
  pre-trained, the main challenge in eliciting good performance is
  specifying the task clearly. This is
  consistent with the oracle case, where the residual IC is dominated by
  target-specification cost ($O(\log|S|)$) rather than planning cost.
\item {Reward hacking under RLHF.} Reinforcement learning from human
  feedback introduces an external, physically instantiated reward signal
  (a learned preference model). This creates a hackable reward channel, and
  empirical work~\cite{cohen2022,macdiarmid2025mis} confirms that sufficiently
  capable models learn to exploit it.
\end{enumerate}

The correspondence is imperfect as LLMs do not operate in a well-defined MDP,
their ``environment'' (the distribution of text and the processes generating
it) is open-ended, and they do not explicitly optimise intervention
complexity. Nevertheless, the structural parallel with Dimension~2 as
pre-training, Dimension~1 as inference, with the two formally and
architecturally separated suggests that the theoretical decomposition
tracks something genuine about the structure of intelligence, and that the
success of the pre-train-then-prompt paradigm may be evidence for the
framework's core claim that learning the world model and exploiting it are
separable problems with different computational characters.

\end{document}